%% file: anonymous-submission-latex-2026.tex
\documentclass[letterpaper]{article} % DO NOT CHANGE THIS
\usepackage{aaai2026}  % DO NOT CHANGE THIS
\usepackage{times}  % DO NOT CHANGE THIS
\usepackage{helvet}  % DO NOT CHANGE THIS
\usepackage{courier}  % DO NOT CHANGE THIS
\usepackage[hyphens]{url}  % DO NOT CHANGE THIS
\usepackage{graphicx} % DO NOT CHANGE THIS
\usepackage{natbib}  % DO NOT CHANGE THIS AND DO NOT ADD ANY OPTIONS TO IT
\usepackage{caption} % DO NOT CHANGE THIS AND DO NOT ADD ANY OPTIONS TO IT
\usepackage{algorithm}%
\usepackage{algorithmicx}%
\usepackage{algpseudocode}%
\usepackage{listings}%
\algtext*{EndIf}
\algtext*{EndFor}
\algtext*{EndWhile}

\usepackage{newfloat}
\usepackage{listings}
\DeclareCaptionStyle{ruled}{labelfont=normalfont,labelsep=colon,strut=off} % DO NOT CHANGE THIS
\floatstyle{ruled}
\newfloat{listing}{tb}{lst}{}
\floatname{listing}{Listing}
\nocopyright % -- Your paper will not be published if you use this command
\title{Subgoaling Relaxation-based Heuristics for Numeric Planning with Infinite Actions}
\author {
    Ángel Aso-Mollar\textsuperscript{\rm 1},
    Diego Aineto\textsuperscript{\rm 1},
    Enrico Scala\textsuperscript{\rm 2},
    Eva Onaindia\textsuperscript{\rm 1}
}
\affiliations {
    \textsuperscript{\rm 1}Universitat Politècnica de València\\
    \textsuperscript{\rm 2}Università degli Studi di Brescia\\
    \{aaso,dieaigar,onaindia\}@vrain.upv.es, \{enrico.scala\}@unibs.it
}

\usepackage{bibentry}
\usepackage{amsfonts}
\usepackage{amsmath}
\usepackage{amsthm}

\newtheorem{definition}{Definition}
\newtheorem{theorem}{Theorem}

\newtheorem{prop}{Proposition}

\newcommand{\pluseq}{\mathrel{+}=}

\begin{document}

\maketitle

\begin{abstract}

Numeric planning with control parameters extends the standard numeric planning model by introducing action parameters as free numeric variables that must be instantiated during planning. This results in a potentially infinite number of applicable actions in a state. In this setting, off-the-shelf numeric heuristics that leverage the action structure are not feasible. In this paper, we identify a tractable subset of these problems—namely, controllable, simple numeric problems—and propose an optimistic compilation approach that transforms them into simple numeric tasks. To do so, we abstract control-dependent expressions into bounded constant effects and relaxed preconditions. The proposed compilation makes it possible to effectively use subgoaling heuristics to estimate goal distance in numeric planning problems involving control parameters. Our results demonstrate that this approach is an effective and computationally feasible way of applying traditional numeric heuristics to settings with an infinite number of possible actions, pushing the boundaries of the current state of the art.

% \textcolor{blue}{. This represents a significant advancement over the current state of the art.}

% \textcolor{blue}{marking a notable step forward in the state of the art.}

% \textcolor{blue}{constituting a major improvement on existing state-of-the-art methods.}

% \textcolor{blue}{pushing the boundaries of the current state of the art.}

\end{abstract}


\input{macros}

\section{Introduction}

In numeric planning, states include numeric variables, which are updated by actions through arithmetic effects \cite{FoxL03}. This extension of classical planning enables quantitative information to directly influence the evolution of the system, representing phenomena such as resource consumption, cumulative change, or behavior governed by numeric computations.

Support for numeric state variables is insufficient for many problems in which the control is subject to numerically parameterized actions. Consider, for example, a \texttt{turn-right} action with a certain degree of rotation or a \texttt{pour-water} action involving an arbitrary number of liters. Research conducted in this line focuses on extending numeric planning to include \textit{control parameters}; i.e., action parameters that extend over infinite numeric domains. %, taking either discrete or continuous values. 
Different approaches have explored this idea from complementary perspectives. The TM-LPSAT planner \cite{ShinD05} integrates the control parameters into a hybrid SAT and linear programming framework. POPCORN \cite{SavasFLM16} and NextFLAP \cite{nextflap24}, in contrast, embed control parameters within a forward partial-order planning search. More recently, the approach of \cite{heesch24} delegates the selection of control parameter values to a neural model.

%The aforementioned works treat control parameters as constraints that narrow the search space; and, due to the infinite nature of the parameters, these approaches rule out modeling them as decision variables. \textcolor{magenta}{
The aforementioned works treat control parameters as constraints that narrow the search space, ruling out the modeling of these parameters as decision variables, as this would lead to an infinite action space.
For example, a robot may be allowed to turn right through an angle between 20° and 45°, yet the consequences of each specific value can differ significantly. 
The S-BFS approach \cite{aso-mollar2025ICAPS}, however, introduces sampling into a forward state-space search algorithm to explicitly handle control parameters during planning. 

% DPEX studies the problem of numeric planning with \textit{control variables}, a reformulation of actions with infinite domain numeric parameters. It was introduced as a principled framework for reasoning in such settings under full state information, thereby enabling the use of heuristic functions as estimators. DPEX provides a systematic way of searching when there are infinitely many action instantiations, using the concept of Delayed Partial EXpansion, that is, using a sampling function to iteratively generate subsets of successors. It has been shown to be competitive with respect to the state of the art even with simple heuristics \cite{aso-mollar2025IJCAI}. The core difficulty of this approach, however, lies in the lack of strong and suitable heuristic estimators for problems with control variables, since standard numeric heuristics \cite{Scala2016IntervalBasedRF, subgoaling20} cannot be utilized in this setting. This is due to the fact that control variables give rise to infinite actions, one for each numeric parameter instantiation.

S-BFS studies the problem of numeric planning with \textit{control variables}, a reformulation of actions with infinite domain numeric parameters. It is a principled framework for reasoning in such settings under full state information, which allows for the use of heuristic functions as estimators. S-BFS provides a systematic way of searching when there are infinitely many action instantiations, using a sampling function to iteratively generate subsets of successors. Although it shows to be competitive with respect to SOTA methods, even when simple heuristics are used \cite{aso-mollar2025IJCAI}, the core difficulty with this approach lies in that standard, informative numeric heuristics \cite{Scala2016IntervalBasedRF, subgoaling20} cannot be used in infinite action spaces.

In this paper, we identify a tractable fragment of numeric planning with control variables and introduce an optimistic compilation that transforms such problems into simple numeric tasks, a common approach in the literature of numeric planning \cite{scala18} or HTN planning \cite{htncompilation} for defining novel estimators. To do so, we abstract control-dependent expressions and convert them into bounded constant effects and relaxed preconditions. We prove that the resulting compiled problems are safe pruning under the subgoaling relaxation. This enables existing numeric subgoaling heuristics to be used directly to estimate the distance to the goal in the original infinite tasks. Our results show that this compilation yields an effective and computationally feasible mechanism for applying traditional numeric heuristics in problems with infinite actions. 

\section{Background}

In this section, we summarize the control variables formalism of S-BFS, its search scheme and the subgoaling relaxation for numeric planning, on which the heuristic estimators defined by our approach are based.

%Firstly, we provide a brief overview of the planning setting in which our work is framed. We first revisit the formulation of numeric planning with control variables \textcolor{magenta}{you assume the framework of control variables as a standard approach ... reference to DPEX, probably better to say we revisit the control variable approach of DPEX ... if you say 'control variables' is DPEX}, which extends classical numeric planning by incorporating actions whose parameters may range over infinite domains. We then recall the DPEX search scheme, a forward state-space algorithm designed to cope with the resulting infinite branching in a systematic way, leveraging heuristic information \textcolor{magenta}{??? DPEX leverages heuristic information?? or is the approach presented here?? This is background ...}. Finally, we summarise the subgoaling relaxation for numeric planning, on which the heuristic estimators defined by our approach are based.

\subsection{Numeric planning with control variables}

We adopt the numeric planning with control variables formalization by \citet{aso-mollar2025IJCAI} and adapt it to our needs. In this setting, $F$ is a set of propositional state variables, $X$, a set of numeric state variables, and $U$ a set of numeric control variables. We assume for every variable $v \in X \cup U$ that a valuation function maps $v$  to its domain $Dom(v)$, and that a numeric domain is a subset of the rational numbers; that is, $Dom(v)\subseteq\mathbb{Q}$ for every $v\in X\cup U$. We denote the set of arithmetic expressions over $X$ and $U$ as Expr$(X\cup U)$. First, we define controllable numeric conditions as  usual numeric planning conditions that also involve control variables.
\begin{definition}[Controllable numeric condition]
    A \textbf{controllable numeric condition} is an inequality $(\xi \bowtie 0)$, where $\xi \in \textup{Expr}(X\cup U)$ and $\bowtie \in \{<, \leq, \doteq, \geq, >\}$, combined with logical operators. We denote the set of all controllable numeric conditions over $X$ and $U$ as $\textup{Constr}_\mathbb{Q}(X \cup U)$. 
\end{definition}

Controllable numeric assignments are defined as numeric assignments that involve control variables in the right part of the assignment.

\begin{definition}[Controllable numeric assignments]
    A \textbf{controllable numeric assignment} is an atomic update of the form $(x := \xi)$, where $x \in X$ and $\xi \in \textup{Expr}(X\cup U)$. We denote assignments of the form $(x := x + \xi)$ as $(x \pluseq \xi)$. We also denote the set of all controllable numeric assignments from $X\cup U$ to $X$  as $\textup{Assign}_\mathbb{Q}(X, U)$. A set of numeric assignments $\mathcal{A} \subseteq \textup{Assign}_\mathbb{Q}(X, U)$ is \textbf{consistent} if, for every $x \in X$, there is at most one assignment $(x := \xi)$ in $\mathcal{A}$, as defined for propositional assignments.
\end{definition}

Additionally, we denote as $\textup{Constr}_\mathbb{B}(F)$ and $\textup{Assign}_\mathbb{B}(F)$ the sets of propositional conditions and assignments, respectively.

\begin{definition}[Numeric planning problem with control variables]
A \textbf{numeric planning problem with control variables} is a tuple $\mathcal{P} = (F, X, U, A, s_0, G)$, where:
    \begin{itemize}
        \item $A$ is a finite set of actions $a=(\textup{Pre}(a),\textup{Eff}(a))$, where $\textup{Pre}(a)=\langle\textup{Pre}_\mathbb{B}(a),\textup{Pre}_\mathbb{Q}(a)\rangle$ and $\textup{Eff}(a)=\langle\textup{Eff}_\mathbb{B}(a),\textup{Eff}_\mathbb{Q}(a)\rangle$, such that
        \begin{enumerate}
        \item $\textup{Pre}_\mathbb{B}(a) \in \textup{Constr}_\mathbb{B}(F)$ and $\textup{Pre}_\mathbb{Q}(a) \in \textup{Constr}_\mathbb{Q}(X\cup U)$ are sets of conditions;
        \item $\textup{Eff}_\mathbb{B}(a) \subseteq \textup{Assign}_\mathbb{B}(F)$ and $\textup{Eff}_\mathbb{Q}(a) \subseteq \textup{Assign}_\mathbb{Q}(X,U)$ are consistent sets of assignments;
        \end{enumerate}
        \item $s_0$ is the initial state with valuations over $F$ and $X$;
        \item $G = \langle G_\mathbb{B}, G_\mathbb{Q} \rangle$, where $G_\mathbb{B} \in \textup{Constr}_\mathbb{B}(F)$ and $G_\mathbb{Q} \in \textup{Constr}_\mathbb{Q}(X)$, are the goal conditions. 
    
    \end{itemize}
\end{definition}

A \textbf{plan} in this setting, $\pi=(\langle a_i,\mu_i\rangle)_{i=1}^k$, is a sequence of pairs consisting of an action and a control valuation. A plan is \textbf{valid} if there exists a sequence of states $(s_i)_{i=1}^{k}$ such that $s_i \models subs_{\mu_i}(\text{Pre}(a_i))$, where $subs_{\mu_i}$ denotes the substitution of control variables in numeric expressions according to valuation $\mu_i$. A plan is a \textbf{solution} if it leads to a goal state, i.e., if $s_k \models G$. %Recall that a \textbf{numeric planning problem} is a numeric planning problem with control variables such that $U=\emptyset$; thus, we use the same notation as in this section for numeric planning problems. \textcolor{magenta}{??????????? what are you trying to say here?? you use the same notation where??? -> lo pongo en subgoaling}

\subsection{Searching with Delayed Partial EXpansions}

Hereinafter, we approach the problem of searching for a valid plan using the search scheme of S-BFS, which is based on Delayed Partial Expansions (DPEX) of search nodes. To cope with the infinite branching factor, DPEX follows a BFS-like scheme that relies on a sampling function $\phi$ to iteratively generate finite subsets of successors, referred to as partial expansions, and on a rectification function $r_h$ to update the parent's $f$-value after each partial expansion. The rectification function $r_h$ can be seen as an abstraction of a heuristic function $h$ that adjusts the evaluation of $h$ accordingly to the number of partial expansions performed in a given state.
%combines best-first search interleaved with sampling and rectification \textcolor{magenta}{you have to give a descriptive definition first, say briefly what is sampled and what rectification is ... better to give first an intuitive and descriptive exaplanation}. 
Algorithm~\ref{alg:algorithm} describes how DPEX works. 

\begin{algorithm}[h!]
\caption{DPEX$_{\phi,r_h}$}
\label{alg:algorithm}
\begin{algorithmic}[1]
\State \textbf{Input:} Sampling function $\phi$, rectification function $r_h$, num. samples $K$, initial state $s_0$, goal conditions $G$
\State \textbf{Output:} Goal reached
\State $Open \gets \{(f(s_0), s_0)\}$
\While{$Open \neq \emptyset$}
    \State Extract node $s$ with lowest $f$-value from $Open$
    \If{$s \models G$}
        \State \Return true
    \Else
        \For{$i = 1$ \textbf{to} $K$}
            \State Sample a successor $s'$ using $\phi$
            \State Insert $(f(s'), s')$ into $Open$
        \EndFor
        \If{$s$ \textbf{is not fully expanded}}
            \State Rectify $f(s)$ using $r_h$
            \State Insert $(f(s), s)$ into $Open$
        \EndIf
    \EndIf
\EndWhile
\State \Return false
\end{algorithmic}
\end{algorithm}

The algorithm maintains a frontier of states ordered by an evaluation function $f$ (line 3). At each iteration, the state with lowest $f$ is selected (line 5) and $K$ successors  are sampled using a sampling function $\phi$ (lines 9–10) and inserted in the frontier (line 11). If the selected node is not fully expanded, its $f$-value is updated via a rectification function $r_h$ (lines 12–13) and reinserted (line 14). Through this interplay of sampling and rectification, DPEX can handle infinite branching factors efficiently while remaining probabilistically complete under mild assumptions on $\phi$ and $r_h$. It becomes critical to have a heuristic function $h$ that accurately estimates the cost to the goal. Hereinafter, we will attempt to address this issue by leveraging the subgoaling relaxation.

\begin{figure*}[ht]
    \centering
    \includegraphics[width=0.7\linewidth]{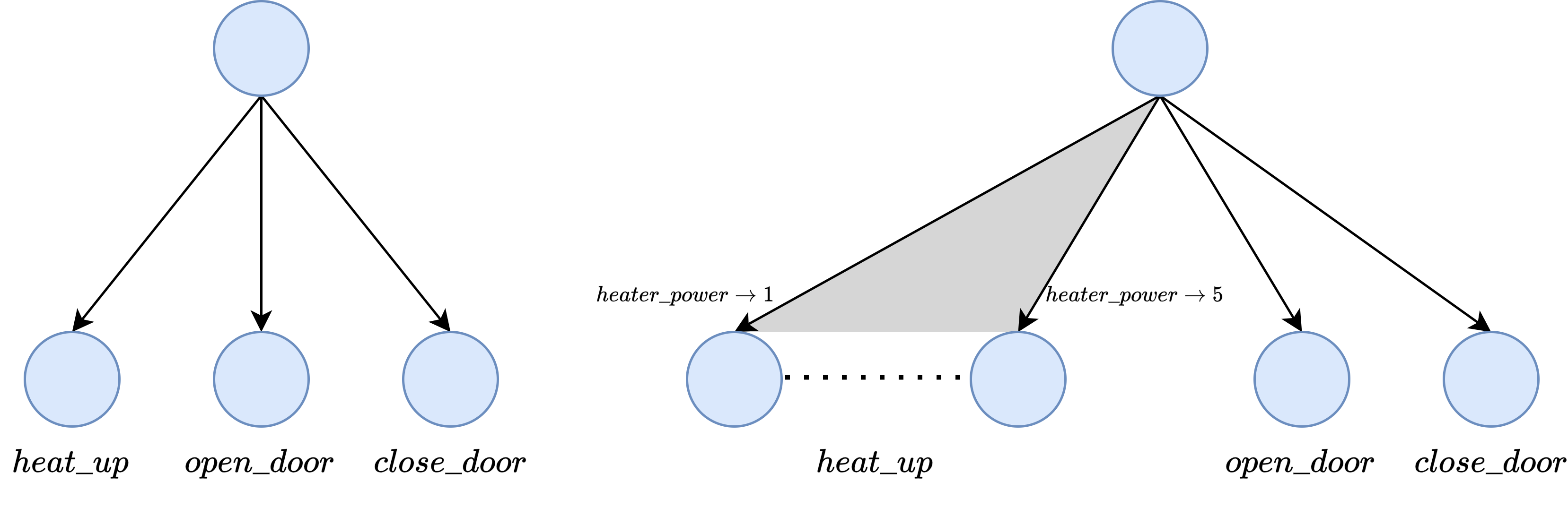}
    \caption{The inclusion of control variables introduces a shift in the decision space of numeric planning problems.}
    \label{fig:decisionspace}
\end{figure*}

\subsection{Subgoaling relaxation for numeric planning}

A \textbf{numeric planning problem} is a numeric planning problem with control variables such that $U=\emptyset$. For numeric planning problems, the subgoaling relaxation extends its classical counterpart \cite{bonet01} by taking advantage of the numeric structure of the problem. It relies on the notion of the m-times regressor \cite{subgoaling20}, an operator that rewrites a condition to identify the actions that can contribute positively to achieving it. Given a state $s$, an action is considered a \textbf{possible achiever} of a condition if $s$ satisfies the $m$-regressed version of that condition for some $m\in \mathbb{N}$.

The subgoaling relaxation is particularly effective for simple numeric planning problems, that is, problems in which all conditions are linear and all assignments follow the form $x \pluseq k$ for some constant $k$. In these cases, the m-times regressor can be computed in closed form. This relaxation allows the derivation of both admissible and inadmissible heuristic estimates, following the principles of the classical $h_{max}$ and $h_{add}$ heuristics, respectively. Such heuristics can be computed by using just possible achievers.  

Let $\psi$ be a condition and $\psi^{r(a,m)}$ the m-times regressor of $\psi$ through $a$; the $add$ heuristic is defined as follows

\;

\noindent$
\hat{h}^{add}_{hbd+}(s, \psi) = \\
\begin{cases}
0 
    & \text{if } s \models \psi \\[6pt]

\displaystyle \min_{a \in ach(\psi)} 
\left( \hat{h}^{add}_{hbd+}(s, \mathrm{Pre}_\mathbb{B}(a)) + \lambda(a) \right)
    & \psi \in \mathrm{PCs} \\[10pt]

\displaystyle \min_{\substack{a \in A,\\ s \models \psi^{r(a,\hat{m})},\\ \forall \hat{m}\in \mathbb{Q}^+}}
\left( \hat{m}\,\lambda(a) + \hat{h}^{add}_{hbd+}(s, \mathrm{Pre}_\mathbb{Q}(a)) \right)
    & \psi \in \mathrm{SCs} \\[12pt]

\displaystyle \sum_{\substack{c \subset \psi \\ |c| = 1}}
\hat{h}^{add}_{hbd+}(s, c)
    & |\psi| > 1
\end{cases}
$

\;

Where $\lambda(a)$ denotes the cost of action $a$, PC is the set of propositional conditions and SC is the set of simple numeric conditions. $\hat{m}$ represents the continuous relaxation of the number of repetitions, used for computational complexity reasons \cite{subgoaling20}. 

Let $\P$ be a numeric planning problem; we denote its subgoaling relaxation by $\P^1$. We are particularly interested in proper \textbf{relaxations} that guarantee an over-approximation of the problem; i.e., if the original problem has at least one solution, so does its relaxation. This property proves very beneficial in the heuristic search, as it can be used for safely pruning and tends to provide more reliable estimates. 

\section{Controllable simple numeric problems}

Extending numeric planning problems with control variables makes the decision space of each state potentially infinite, since now a single action with a control variable gives rise to infinitely many grounded instantiations, i.e., pairs of action-control valuation. Figure~\ref{fig:decisionspace} illustrates this problem in a node with three applicable actions.

The infinitude of the action space directly impacts the heuristic computation. Recall that, in the subgoaling relaxation, an action is considered a possible achiever of a condition when the regressed condition holds in the current state. With control variables, a single condition may therefore have infinitely many possible achievers, making the explicit computation of the heuristic derived from the relaxation infeasible. 

A common strategy in related work facing similarly prohibitive costs is to compile the original problem into a simpler numeric model that relaxes certain aspects of the dynamics. Heuristic estimates are then computed over this compiled problem, which is more amenable to explicit reasoning and efficiently captures an informative approximation of the original task. This compilation-based approach is recurrent in the literature and can be seen, for instance, in work such as the effect-abstraction based relaxation in \cite{scala18} for numeric planning problems with linear effects or the HTN planning abstraction for guiding search \cite{htncompilation}.

Following this strategy, we seek to identify a subset of numeric planning problems with control variables whose compilation yields simple numeric planning problems. Since, in standard numeric planning \textit{without} control variables, the subgoaling relaxation admits a closed-form computation only for the class of simple numeric planning problems \cite{subgoaling20}, our goal is to characterize an analogous class in the control-variable setting, which we call controllable simple numeric problems, with the aim of enabling the effective use of the subgoaling relaxation in the control-variable setting through a compilation.

%\textcolor{blue}{que quede claro que estoy haciendo una identificación deun subconjunto de problemas}This motivates the need for a compilation that brings numeric planning with control variables into a form where the subgoaling relaxation from numeric planning can be effectively used. For this reason, and given that the subgoaling relaxation is best suited for simple numeric planning problems, we aim to \textcolor{blue}{identify} define a subset of numeric planning problems with control variables such that their compiled version becomes simple numeric.

\begin{definition}[Controllable simple numeric conditions]
    Given a set of controllable numeric assignments $ \mathcal{A} \subset \textup{Assign}_\mathbb{Q}(X,U)$, a numeric condition $(\xi \rhd 0)$ is \textbf{controllable simple} iff \begin{itemize}
        \item $\xi$ is linear w.r.t. $X$, i.e., $\xi= (\xi_X \rhd \xi_U)$, where $\xi_X\in \text{Expr}(X)$ is linear and $\xi_U \in \textup{Expr}(U)$;
        \item $\rhd\in\{>,\geq\}$;
        \item $\forall\ \varphi \in \mathcal{A}$ such that $\varphi = (x := \xi')$, if $x$ appears in $\xi$, $\varphi$ can be rewritten as $(x \pluseq \xi'')$, where $\xi'' \in \textup{Expr}(U)$. 
    \end{itemize}  
\end{definition}

A numeric planning problem with control variables $\P=(F,X,U,A,s_0,G)$ is a \textbf{controllable simple numeric problem}  iff every condition in $\P$ is controllable simple. Recall that the set of all controllable simple numeric problems is a subset of the set of all numeric planning problems with control variables.

\begin{example}
    Consider a controllable simple numeric problem defined as follows:
    \begin{itemize}
        \item Fluents: $F = \emptyset$,
        \item Numeric state variables: $X = \{x, y\}$,
        \item Control variables: $U = \{u_1, u_2\}$,
        \item Initial state: $x = 0$, $y = 0$,
        \item Numeric goal: $x > 20$,
        \item Single action $a$ with:
        \begin{itemize}
            \item Preconditions: $x > u_1$, $y > u_2$,
            \item Effects: $x \pluseq 2u_1 + u_2$, $y \pluseq u_1 - 3u_2$.
        \end{itemize}
    \end{itemize}
    
    This problem is a controllable simple numeric problem because the goal and conditions are controllable simple, and the effects of the action are of the desired form, i.e., additive and linear only with respect to the control variables.
\end{example}

Controllable simple numeric problems are problems such that conditions are linear in the state variables and such that the effects depend only on control variables. Since control variables are \textit{free and bounded} variables that are not part of the state, we can consider them as ``lifted numbers''. From this perspective, concretizing the values of control variables in a controllable simple numeric planning problem yields a simple numeric planning problem. What we need is a concretization that makes the relaxation safe-pruning.

\section{Optimistic compilation}

In this section, we present a relaxation that compiles a controllable simple numeric planning problem into a simple numeric one. We demonstrate that this compilation is safe-pruning under the subgoaling relaxation, which automatically allows the use of numeric subgoaling heuristics for estimate calculations through the compilation. To lay the groundwork, we first briefly review the closed arithmetic of intervals:

\begin{definition}[Closed arithmetic of intervals]
    Given $x=[\underline{x},\overline{x}]$ and $y=[\underline{y},\overline{y}]$, then:
    \begin{itemize}
        \item $x+y = [\underline{x}+\underline{y},\overline{x}+\overline{y}]$
        \item $x-y = [\underline{x}-\overline{y},\overline{x}-\underline{y}]$
        \item $x\cdot y = [\min(\underline{x}\cdot\underline{y},\underline{x}\cdot\overline{y},\overline{x}\cdot\underline{y},\overline{x}\cdot\overline{y}),\max(\underline{x}\cdot\underline{y},\underline{x}\cdot\overline{y},\overline{x}\cdot\underline{y},\overline{x}\cdot\overline{y})]$
    \end{itemize}
\end{definition}

Given an arithmetic expression over $U$, $\xi \in \text{Expr}(U)$, we will denote $Dom(\xi)$ as the interval resultant from the arithmetic operation of the domain interval for every variable. For example, if $Dom(u_1)=[1,2]$ and $Dom(u_2)=[0,3]$, then $Dom(3u_1+u_2)=3\cdot[1,2]+[0,3]=[3,6]+[0,3]=[3,9]$. We will refer to the lower and upper bound of $Dom(\xi)$ as its \textbf{extreme valuations}, i.e., $\underline{Dom(\xi)}$ and $\overline{Dom(\xi)}$.

We define the optimistic compilation for the controllable simple numeric problem as a new problem with every controllable simple condition relaxed to its lower bound, and considering every extreme valuation for every effect of the original actions, that is:

\begin{definition}[Optimistic compilation]
\label{opt-compilation}
    Given a controllable simple numeric planning problem $\mathcal{P} = (F, X, U, A, s_0, G)$, we define the \textbf{optimistic compilation} of $\mathcal{P}$, $\mathcal{P}_O=(F,X,A_O,s_0,G)$, as the numeric planning problem resulting from substituting every expression $\xi\in\text{Const}(X\cup U)$ with its extreme valuations. That is, for each $a \in A$ such that $\text{Eff}_\mathbb{Q}(a)=\{(x_1\pluseq \xi^{x_1}_u),\dots,(x_n\pluseq\xi_u^{x_n})\}$ for some $n\in\mathbb{N}$, $x_1,\dots,x_n\in X$ and $\xi_u^{x_1},\dots,\xi_u^{x_n}\in\text{Expr}(U)$ we define a new set $\Lambda(a)$ such that:
    \begin{itemize}
        \item $\Lambda(a)=\{a_{\lambda_1,\dots,\lambda_n} \mid \lambda_i \in \left\{\underline{Dom(\xi_u^{x_i})},\overline{Dom(\xi_u^{x_i})}\right\}\;$ $\left.\forall i \in \{1,\dots,n\}\right\}$,
        \item $\text{Eff}_\mathbb{Q}(a_{\lambda_1,\dots,\lambda_n})=\{ (x_1\pluseq\lambda_1),\dots,(x_n\pluseq\lambda_n)\}$
        \item For every atomic condition $\psi\equiv(\xi_X \rhd \xi_U)$ appearing in $\text{Pre}_\mathbb{Q}(a)$, it transforms to $\psi_O\equiv(\xi_X \rhd \underline{Dom(\xi_U)})$ in $\text{Pre}_\mathbb{Q}(a_{\lambda_1,\dots,\lambda_n}))\; \forall \lambda_1,\dots,\lambda_n$.
        \item $A_O=\bigcup_{a\in A} \Lambda(a)$
    \end{itemize}

\end{definition}

This compilation produces a finite problem by exhaustively enumerating all optimistic valuations of the control-dependent effects of each action. In parallel, all controllable simple conditions in the preconditions are relaxed to their lower bound, intuitively ensuring that any action that is executable in the original problem under some valuation of the control variables remains executable in the compiled task. The resulting problem then over-approximates the behavior of the original one.

\begin{example}
    Consider Example 1 from the previous section, with control variable domains $Dom(u_1)=[0,4]$ and $Dom(u_2)=[3,5]$. First, we compute the ranges of the linear expressions:
    \[
        Dom(2u_1 + u_2) = [3,13], \quad Dom(u_1 - 3u_2) = [-15,-5].
    \]
    The optimistic compilation generates four actions corresponding to the extreme bounds of these expressions. All actions share the same preconditions, coming from the original preconditions $x > u_1$, $y > u_2$:
    \[
        x > 3, \quad y > -15.
    \]
    The actions and their effects are:

    \begin{center}
    \begin{tabular}{c|c|c}
        Action & Effect on $x$ & Effect on $y$ \\
        \hline
        $a_{3,-15}$ & $x \pluseq 3$ & $y \pluseq -15$ \\
        $a_{13,-15}$ & $x \pluseq 13$ & $y \pluseq -15$ \\
        $a_{3,-5}$ & $x \pluseq 3$ & $y \pluseq -5$ \\
        $a_{13,-5}$ & $x \pluseq 13$ & $y \pluseq -5$ \\
    \end{tabular}
    \end{center}

    Each action represents one combination of extreme values for the expressions involving the control variables.%, allowing the compilation to conservatively approximate the effects of the original action.
\end{example}
The first thing we need to prove is that, in fact, this compilation induces a simple numeric planning problem.

\begin{theorem}
    Given $\mathcal{P} = (F, X, U, A, s_0, G)$ a controllable-simple numeric planning problem, $\mathcal{P}_O$ is a simple numeric planning problem.
\end{theorem}
\begin{proof}
    This holds trivially since control variables only appear in actions, and since $\mathcal{P}$ is controllable-simple, conditions are linear with respect to $X$, so in $\mathcal{P}_O$ conditions are linear given that control variables' expressions become constants. Effects also become additive with respect to a constant for the same reason.
\end{proof}

After proving that the compiled problem is in fact simple numeric, we need to prove that the optimistic compilation is safe pruning under the subgoaling relaxation. This will enable an effective use of the subgoaling relaxation as dead-end detector and as heuristic estimator.

\begin{theorem}
\label{thm:relaxation-bruteforce}
    Given $\mathcal{P} = (F, X, U, A, s_0, G)$ a controllable-simple numeric planning problem, let $\Pi$ be the set of solutions for $\P$. Then $\Pi_\mathcal{P} \neq \emptyset \implies \Pi_{\mathcal{P}_O^1} \neq \emptyset$, i.e., the compilation is safe pruning under the subgoaling relaxation.
\end{theorem}

\begin{proof}
    If $\Pi_\mathcal{P} \neq \emptyset$, let $\pi\in\Pi_\mathcal{P}$ be a solution for $\mathcal{P}$. We proceed by induction over $\pi$.

    \paragraph{Base case.} A condition reached by an empty plan is trivially satisfied in the initial state $s_0$, which is the same in both compilations.

    \paragraph{Inductive step.} A condition reached by $\pi$ that is not true in the initial state implies the existence of some pair $\langle a_i,\mu_i\rangle$ in $\pi$ making $\psi$ true from the state $s_{i-1}$ reached by the prefix of $\pi$ up to $\langle a_i,\mu_i\rangle$. We want to discover a possible achiever in the relaxation that works towards $\psi$. Let us define $\psi$ as an abstract condition $\psi\equiv \sum_{x_i}w_{x_i}\cdot x_i\;  \rhd\; \xi_u$, where $\xi_u\in\text{Expr}(U)$. But $ \sum_{x_i \in X}w_{x_i}\cdot x_i \rhd \xi_u \geq \underline{Dom(\xi_u)}$, and thus the set of states that fulfill the compiled condition in $\mathcal{P}_O$ is always greater than those in $\mathcal{P}$. In particular, since $s_{i-1}\models_{\mu_i} \text{Pre}(a_i)$, by hypothesis of induction, then it follows that $s_{i-1} \models \text{Pre}({a_i}_{\lambda_1,\dots,\lambda_n})\; \forall \lambda_1,\dots,\lambda_n$. 
    
    Now, it is just a matter of choosing the right lambdas in order for the effects to work towards fulfilling $\psi$. This can be computed by calculating the net effect $N_{\psi,a}$ of each action ${a_i}_{\lambda_1,\dots,\lambda_n}$ and choosing the right extreme values for it to be positive. For every $e \in \text{Eff}(a)$, the net effect before the compilation, or the contribution of the action towards fulfilling the condition, can be computed as a trivial generalization of the net effect for standard numeric planning \cite{subgoaling20}:    $$N_{\psi,a}=\sum_{\substack{x_i\in lhs(e) \\ e \in \text{Eff(a)}}}w_{x_i}\cdot \xi_u^{x_i´}$$
    But $N_{\psi,a}$ in this setting is an expression that depends on control variables; concretely, its domain can be calculated as an interval regarding the close arithmetic:
$$Dom(N_{\psi,a})=\sum_{\substack{x_i\in lhs(e) \\ e \in \text{Eff(a)}}}w_{x_i}\cdot Dom(\xi_u^{x_i´})$$
    If the domain of the net effect is greater than zero, then the action is a possible achiever for the condition because it has a positive contribution on the condition. $Dom(N_{\psi,a})$ is an interval that \textbf{we actually know has a positive side}, because of the inductive hypothesis, since there exists some control valuation $\mu_i$ that actually makes $\psi$ fulfillable, which also means that the substitution of the expression using the control valuation, $subs_\mu(N_{\psi,a})$, follows $subs_\mu(N_{\psi,a})\in Dom(N_{\psi,a})$ and $subs_\mu(N_{\psi,a})>0$. 
    
    We choose lambdas taking into account the sign of each $w_{x_i}$. If $w_{x_i}>0$, we define $\lambda_i:=\overline{Dom(\xi_u^{x_i})}$, and if $w_{x_i}<0$, we define $\lambda_i:=\underline{Dom(\xi_u^{x_i})}$. With this choice of lambdas, we know for certain that $a_{\lambda_1,\dots,\lambda_n}$ is a possible achiever of $\psi_O$. This is because the choice of lambdas also implies a valuation $\mu_O$ such that $subs_{\mu_O}(N_{\psi,a})>=subs_\mu(N_{\psi,a})$, because we are taking the extreme values for the expressions, and we know that the right side of that inequality is greater than zero by induction.

    Finally, observe that the subgoaling relaxation considers reachability of conditions separately, and thus, since all conditions are reachable, so is any conjunction of them, q.e.d.
\end{proof}

Although the optimistic compilation constitutes a safe-pruning relaxation, its main drawback is that the number of compiled actions grows exponentially with the number of effects in each action. The remainder of this work focuses on mitigating this overhead.

\section{Reducing the complexity of the optimistic compilation}

The exponential nature of the optimistic compilation makes it impractical to use. Nevertheless, if we follow the proof of Theorem 2, it is actually possible to reduce the set of optimistic actions involved. The objective of this section is thus to reduce the size of the compilation by looking and theoretical properties that can be inferred from the proof of the last theorem. Let us first define what we call the sign choice function:

\begin{definition}[Sign choice function]
    Let $\mathcal{P}=(F,X,U,A,s_0,G)$ be a controllable simple numeric planning problem and let $a\in A$. Let $\text{Eff}_\mathbb{Q}(a)=\{e_1,\dots,e_m\}$ and for each effect $e_j$ let $lhs(e_j)\in X$ be the state variable it updates and $\xi_u^{lhs(e_j)}\in\text{Expr}(U)$ the control-only expression contributing to that update.
    For a given atomic condition $\psi\equiv\sum_{x\in X} w_x x \ \rhd\ \xi_u^\psi$, we define the \textbf{sign choice} function of a given effect as
    $$
    \chi_{\psi,a}(e_j)=\begin{cases}
        \overline{Dom(\xi_u^{lhs(e_j)})}, & \text{if } w_{lhs(e_j)}>0 \\
        \underline{Dom(\xi_u^{lhs(e_j)})}, & \text{if } w_{lhs(e_j)}<0 \\
        0, & \text{if } w_{lhs(e_j)}=0
    \end{cases}
    $$
    The \textbf{signature} of action $a$ relative to a condition $\psi$ is a vector of sign choices $\sigma_{\psi,a}$ such that:
    $$\sigma_{\psi,a}=(\chi_{\psi,a}(e_1),\dots,\chi_{\psi,a}(e_m))$$
\end{definition}

The signature of an action only uses the variants of the compilation that have a positive net effect over the condition, i.e., that are possible achievers.

\begin{definition}[Signature compilation]
    Let $\mathcal{P}$ be a controllable simple numeric planning problem. Let $\Sigma_a:=\{\sigma_{\psi,a} \mid \psi \in \Psi_a\}$, where $\Psi_a$ is the set of all relevant conditions for $a$, i.e., every condition that contain $lhs(e_j)$ for some $e_j \in \text{Eff}_\mathbb{Q}(a)$. We define the \textbf{signature compilation} of $\mathcal{P}$, $\mathcal{P}_\Sigma$, as the optimistic compilation with a reduced set of actions:
    $$A_\Sigma(a):=\{a_\sigma \mid \sigma \in \Sigma_a\},\quad \forall a \in A$$
    Defining $A_\Sigma = \bigcup_{a\in A}A_\Sigma(a)$, then $\mathcal{P}_\Sigma = (F,X,A_\Sigma,s_0,G)$.
\end{definition}

We can further reduce each set $\Sigma_a$ by collapsing signatures to their non-zero components, that is, by removing effect entries that do not participate in any relevant condition. This allows multiple signatures to be merged when they differ only in components that are irrelevant for all conditions under consideration. For example, the set $\{(0,4),(3,7)\}$ is equivalent to $\{(3,4),(3,7)\}$, and the set $\{(0,1),(5,0)\}$ is equivalent to $\{(5,1)\}$ after collapsing the zero entry. This reduced compilation can also be used to detect non-solvability of the problem, since it exposes actions whose effects do not contribute positively to any goal-relevant condition. The resulting action set $A_\Sigma$ is still a subset of $A_O$ and preserves the relaxation property established in the previous theorem.

\begin{prop}
    Given a controllable-simple numeric planning problem $\mathcal{P}$, then $\mathcal{P}_\Sigma$ is a relaxation under the subgoaling relaxation. In other words, replacing $A_O$ by the reduced set $A_\Sigma$ preserves the relaxation property.
\end{prop}
\begin{proof}
    In the compilation to $\mathcal{P}_\Sigma$, we construct the reduced set $A_\Sigma \subseteq A_O$ by selecting only those instantiations that can act as achievers of at least one goal or subgoal. That is, for every subgoal $g$, if an action $a \in A_O$ can achieve $g$ under some instantiation of its control variables, the corresponding instantiation is included in $A_\Sigma$. Conversely, any action that cannot contribute to achieving any subgoal is discarded.

    Since $A_\Sigma$ is a subset of $A_O$ that preserves all possible achievers for every subgoal, the subgoaling relaxation remains valid: any relaxed plan that exists in $\mathcal{P}$ can also be constructed in $\mathcal{P}_\Sigma$. No new constraints are introduced, and no potential achievers are removed from consideration for the purpose of the relaxation.     Therefore, $\mathcal{P}_\Sigma$ maintains the relaxation property under the subgoaling relaxation, as required.
\end{proof}

\begin{example}
    Continuing the previous example, recall that $a$ is the only action schema with effects
    $e_1: x\pluseq 2u_1+u_2$ and $e_2: y\pluseq u_1-3u_2$, and compiled preconditions $x>3$ and $y>-15$. The conditions relevant to $a$ (that is, those involving a variable appearing in $lhs(e_j)$) are the goal $x>20$ and the numeric preconditions of the action $y>-15$ and $x>3$. We compute the signatures with respect to each of these conditions:

    \begin{itemize}
        \item For $\psi_1\equiv x>20$ and $x>3$ we have coefficients $w_x=1$ and $w_y=0$, and then we only need the upper bound for $x$.

        \item For $\psi_2\equiv y>-15$ we have $w_x=0$ and $w_y=1$, and then we only need the upper bound for $y$.
    \end{itemize}

    The signature set relevant to action $a$ is $\Sigma_a=\{(13,0),(0,-5)\}$, which collapses into $\Sigma_a=\{(13,-5)\}$.
    Hence, the only relevant action is the one with the upper bounds, since $w_x$ and $w_y$ is greater or equal than zero for every precondition. According to the definition of the \emph{signature compilation}, we only need to generate the action $a_{13,-5}$. In this case, the compiled problem has the same amount of actions as the original.

\end{example}

The size of the signature compilation can be proven to be linear with respect to the number of conditions of the problem, i.e., linear with respect to the size of the problem.

\begin{prop}[Size bound of the signature compilation]
    Let $\mathcal{P}$ be a controllable-simple numeric planning problem, and let $\mathcal{P}_\Sigma=(F,X,A_\Sigma,s_0,G)$ be its signature compilation. For each action $a$ from $\mathcal{P}$, let $\Psi_a$ be the set of all relevant atomic conditions for $a$, i.e., every condition that contain $lhs(e_j)$ for some $e_j \in \text{Eff}_\mathbb{Q}(a)$. Let the set of signatures of $a$ be $\Sigma_a := \{\sigma_{\psi,a} \mid \psi \in \Psi_a\}$ and 
    % define the average signature cardinality as
    % $$
    %     \bar{s} := \frac{1}{|A|} \sum_{a\in A} |\Sigma_a|.
    % $$
    % Then the total number of actions in the signature compilation satisfies $|A_\Sigma| = |A| \cdot \bar{s}$. Moreover, 
    let $\Psi$ denote the total number of distinct numeric atomic conditions that appear in $\mathcal{P}$, that is, the number of atomic numeric conditions occurring in both action preconditions and in the goal. Then,
    $$|A_\Sigma| \leq |A| \cdot |\Psi|$$
    which means that the size of the compilation is polynomial with respect to the size of the original problem, in terms of the number of original preconditions.
\end{prop}

\begin{proof}
    By definition $\Sigma_a=\{\sigma_{\psi,a}\mid \psi\in\Psi_a\}$ and hence $|\Sigma_a|=|\Psi_a|$. Since $\Psi_a$ is a subset of the set of all numeric atomic conditions of the problem, we have $|\Psi_a|\leq |\Psi|$ for every $a\in A$. Therefore
    $|A_\Sigma|=\sum_{a\in A} |\Sigma_a| \le \sum_{a\in A} |\Psi| = |A|\cdot |\Psi|$
\end{proof}

The cardinality bound on the action set of the signature compilation remains tractable in practice, thanks to the dependence on the average of signatures between each action and the fact that the set of effects is usually much lower than the set of all variables $X$. The compilation remains almost linear in practice with only higher bounds if tighter sets of preconditions $\Psi_a$ are present, and can be computed polynomially with respect to the number of conditions of the problem. 

We define the heuristic $h^{add}_{\Sigma}$ as the numeric heuristic $\hat{h}^{add}_{hbd+}$ applied to the compiled problem obtained through the signature compilation of the problem.

\section{Extracting more information from subgoaling heuristics}

The additive heuristic $h^{add}_\Sigma$ computes costs by summing relaxed action contributions for each subgoal. While this provides strong and valuable search guidance, this relaxation ignores positive interactions among actions, causing a systematic overestimation of the real plan cost, especially in domains with overlapping effects. This overestimation is a direct consequence of counting the cost of actions multiple times when they contribute to several subgoals, even though, in a real plan, the same action instance may suffice for several achievements.

To address this limitation, the $h^{mrp}$ heuristic \cite{hmrp}, based on multi-repetition relaxed plans in simple numeric planning problems, merges redundant action contributions by tracking the maximum required count for each action rather than simply accumulating all individual contributions. This approach is also safe-prunning and allows $h^{mrp}$ to more accurately capture the dependencies and synergies between numeric actions and subgoals, reducing inadmissible overestimation and providing more informed search guidance for simple numeric planning problems. We then define $h^{mrp}_\Sigma$ as the $h^{mrp}$ heuristic under the signature compilation.

\section{Experiments}

In this section, we analyze the practical impact of the signature compilation used in the $h^{add}_\Sigma$ and $h^{mrp}_\Sigma$ heuristics. We implemented both heuristics in the ENHSP planner and compared their performance with existing heuristics within the Delayed Partial Expansions (DPEX) algorithm. We also evaluated their performance against the NextFLAP planner.

\paragraph{DPEX instantiation}
In our experiments we use DPEX with $f(n)=r_h(n,s)$ as node evaluation criterion, where  $r_h(n,s)=h(s)+\log(1+n)$, i.e., logarithmic rectification, where $n$ is the number of partial expansions of $s$. For each partial expansion, five successors are uniformly sampled via rejection sampling over the hypercube of control variables, rather than directly within the precondition polytope. Direct sampling over the polytope is not a trivial problem, and an interesting direction for future work. Experiments were conducted with a fixed seed on a 12th Gen Intel(R) Core(TM) i9-12900KF CPU running Ubuntu 22.04 LTS, with a 30-minute timeout and 8 GB memory limit, and five runs per problem.

\paragraph{Baselines.} Our experiments are evaluated against baseline search configurations that rely on domain-independent heuristics. Specifically, we compare against two heuristics: a blind  heuristic $h^{0}$ and the Manhattan goal-counting heuristic $h^{mgc}$, which computes the Manhattan distance to the numeric subgoals and counts the number of satisfied propositional subgoals. The Manhattan goal-counting heuristic is particularly effective in domains with numeric goal constraints, although its performance is highly domain specific and does not capture the causal structure of actions. We also compare against a standard run of the deterministic planner NextFLAP, capable of handling control parameters. The purpose of the baseline evaluation is to assess whether exploiting action and goal structure within heuristic computations yields a performance benefit; which is particularly relevant because previous evaluations with DPEX had to rely on structure-independent heuristics. 

\paragraph{Domains.} We use three domains introduced in POPCORN: \textsc{Cashpoint}, \textsc{Procurement} and \textsc{Terraria}. This domains only have propositional goal conditions and are relatively hard, especially the \textsc{Terraria} domain which uses six control variables at the same time and the \textsc{Procurement} domain that has a lot of causality between actions. We also use four enhanced domains from the numeric IPC as in \cite{aso-mollar2025ICAPS}: \textsc{Blockgrouping}, \textsc{Counters}, \textsc{Drone} and \textsc{Sailing}. The first three have numeric goal conditions, while the fourth one has only propositional. The \textsc{Drone} domain is the hardest one since it has a lot of dead-ends, while the \textsc{Sailing} domain stands out for being unbounded in the position of the boats. For each of this seven domains, we considered both continuous (C) and discrete (D) versions, i.e., we used control variables with and without decimals, respectively. For every domain, we generated 20 problems of increasing size, which can be found in the supplementary material.

\paragraph{Heuristic setup.} The overhead cost of setting up the heuristic is negligible, with an average time of approximately 15 milliseconds. Moreover, in Table~\ref{tab:setuptime} we  compare the number of actions with respect to the original problem $A$ and the compiled ones $A_O$ and $A_\Sigma$, and we observe that the overhead of the signature compilation $A_\Sigma$ is far lower than that of the first compilation $A_O$. In certain domains such as \textsc{Cashpoint}, \textsc{Procurement} or \textsc{Terraria} there is only one action needed in the signature compilation, since the control variables act as unbounded amounts of production with no upper limit. Especially in this domains, the difference with respect to the optimistic compilation $A_O$ is clear. As an example, in the \textsc{Procurement} domain every action has an average of 3 control-dependent effects.

\begin{table}[th]
\centering
\small
\begin{tabular}{lccc}
\hline
\textbf{Domain} & $|A|$ & $|A_\Sigma|$ & $|A_O|$  \\
\hline
\textsc{Blockgrouping} & 114.60 & 229.20 & 229.20  \\
\textsc{Cashpoint}      & 7379.00 & 7379.00 & 8758.00  \\
\textsc{Counters}       & 15.00 & 26.00 & 30.00 \\
\textsc{Drone}          & 14.50 & 20.50 & 25.375  \\
\textsc{Procurement}    & 470.00 & 470.00 & 1700.50  \\
\textsc{Sailing}        & 22.70 & 53.45 &  53.45  \\
\textsc{Terraria}       & 40.00 & 40.00 & 120.00 \\
\hline
\end{tabular}
\caption{Average size of action sets $A$, $A_\Sigma$ and $A_O$.}
\label{tab:setuptime}
\end{table}

\begin{figure*}[th]
    \centering
        \includegraphics[width=.99\linewidth]{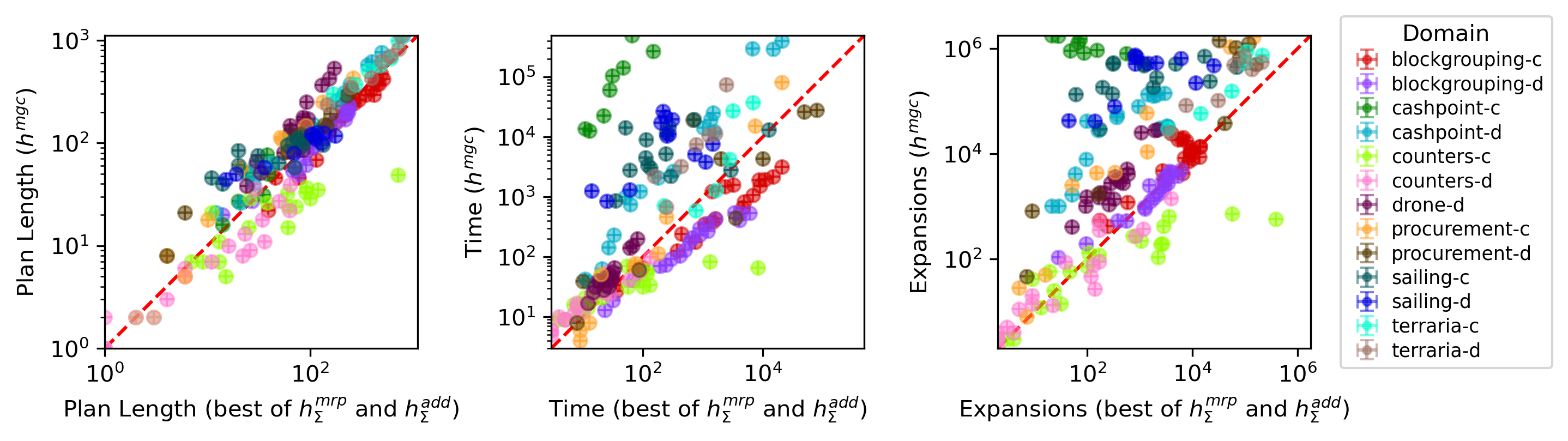}

    \caption{Comparison of the runs of $h^{add}_\Sigma$ or $h^{mrp}_\Sigma$ heuristic versus baselines in terms of plan length, time and number of (partial) expansions, for the best case between the 5 runs.}
    \label{fig:comparison}
\end{figure*}

\paragraph{Results.} Table~\ref{tab:results} reports coverage results for each domain and problem. For this analysis, we consider whether at least one run of DPEX successfully solves the problem. We also show runs solved by either $h^{add}_\Sigma$ or $h^{mrp}_\Sigma$ in the $h_\Sigma$ column, and by NextFLAP in the $NF$ column.

\begin{table}[th]
\centering
\small
\resizebox{\columnwidth}{!}{
\begin{tabular}{l|ccccc|c}
\hline
\textbf{Domain} & $h^{0}$ & $h^{mgc}$ & $h^{add}_\Sigma$ & $h^{mrp}_\Sigma$ & $h_\Sigma$ & $NF$\\
\hline
\textsc{Blockg.-c} (20) & 0  & 20 & 20 & 20 & 20 & 18\\
\textsc{Blockg.-d} (20) & 0  & 20 & 20 & 20 & 20 & 16\\
\textsc{Cashp.-c} (20)     & 10 & 8 & 20 & 20 & 20 & 10\\
\textsc{Cashp.-d} (20)    & 10 & 17 & 20 & 20 & 20 & 9\\
\textsc{Count.-c} (20)     & 12 & 20 & 20 & 19 & 20 & 19\\
\textsc{Count.-d} (20)     & 16 & 20 & 20 & 20 & 20 & 17\\
\textsc{Drone-c} (20)        & 12 & 0  & 2  & 13 & 13 & 0\\
\textsc{Drone-d} (20)        & 18 & 18 & 18 & 18 & 18 & 0\\
\textsc{Procur.-c} (20)  & 8  & 10 & 16 & 10 & 16 & 10\\
\textsc{Procur.-d} (20)  & 5  & 9  & 14 & 11 & 14 & 10\\
\textsc{Sail.-c} (20)      & 0  & 14 & 12 & 20 & 20 & 2\\
\textsc{Sail.-d} (20)      & 1  & 12 & 20 & 20 & 20 & 0\\
\textsc{Terr.-c} (20)     & 6 & 6 & 20 & 4 & 20 & 0\\
\textsc{Terr.-d} (20)     & 6 & 7 & 20 & 3 & 20 & 0\\
\hline
\textsc{Total} (280) & 105 & 181 & 242 & 218 & 261 & 111\\
\hline
\end{tabular}
}
\caption{Coverage for domain and heuristic employed. The column $h_\Sigma$ corresponds to instances solved by either one of $h^{add}_\Sigma$ or $h^{mrp}_\Sigma$, and $NF$ corresponds to instances solved by NextFLAP. 
}
\label{tab:results}
\end{table}

Overall, the compilation-based approach is consistently dominant across all domains. We observe that either the $h^{add}_\Sigma$ heuristic or the $h^{mrp}_\Sigma$ heuristic systematically outperform the $h^{mgc}$ executions. This is particularly noteworthy because $h^{mgc}$ is specifically tailored for purely numeric domains such as \textsc{Blockgrouping}, \textsc{Counters} or \textsc{Sailing}, where the results show that incorporating action structure into heuristic computation is generally beneficial, especially for the \textsc{Sailing} domain.

The assumptions introduced by the compilation’s relaxation directly affect the performance of both heuristics, as reflected in the two columns corresponding to $h^{add}_\Sigma$ and $h^{mrp}_\Sigma$. The compilation purposefully ignores control variables values when computing heuristic estimates. This can result in overly aggressive commitments to actions that do not properly account for control-dependent feasibility conditions. The effect is more pronounced in domains with a larger number of control variables or where these variables play a significant role and participate in many constraints, such as \textsc{Procurement} or \textsc{Drone}, and might explain the performance degradation.

This tendency of the compilation committing too aggressively to certain actions also interacts with the fact that the additive heuristics overestimate the true cost. Interestingly, there appears to be a nontrivial interaction between the compilation’s underestimation and the heuristic’s overestimation. The analysis suggests a correlation between these two effects when looking at coverage in domains such as \textsc{Procurement} or \textsc{Counters}, which perform better in $h^{add}_\Sigma$ runs than in $h^{mrp}_\Sigma$ ones, and it supports the hypothesis that both phenomena may “balance” each other out, partially compensating for their respective flaws and ultimately improving performance. The $h^{mrp}_\Sigma$ correction, although better for the general case, might worsen results in some specific scenarios.

Regarding the comparison with NextFLAP, we observe that its performance is comparable to blind search with $h^0$. This is a relevant result, because our approach exploits full state information to perform heuristic search, and this leads to significantly better scalability in practice, as confirmed by our analysis.

Following a more detailed analysis, Figure \ref{fig:comparison} presents a comparison in terms of plan length, time and number of (partial) expansions for each domain individually regarding $h_\Sigma$ with respect to the best of the baselines $h^{mgc}$. We evaluate only the instances solved by both of the proposed approaches, comparing the best-performing run of either strategy. Extended figures with additional comparisons can be found in the supplementary material.

\;

We observe a clear dominance of our approach over the goal-sensitive strategy, as well as a consistent trend toward improved performance compared to the numeric heuristic $h^{mgc}$. As expected, $h^{mgc}$ remains slightly superior regarding analyzed metrics for the pure numerical domains, such as \textsc{Blockgrouping} or \textsc{Counters}, when considering the instances solved by both approaches. This behavior is likely tied to the nature of the compilation, further motivating the search for more informed ways of exploiting the infinite yet structured action space present in these problems.

\section{Conclusions and future work}

As a conclusion, we have proposed a heuristic based on the numeric subgoaling relaxation that is applicable to the setting of control variables and achieves competitive results compared to standard domain-independent heuristics. This heuristic is based on an optimistic relaxation, which transforms the original problem into a tractable simple numeric planning problem. This compilation, however, is exponential in the number of effects of the problem at hand, so to mitigate this potential combinatorial explosion we introduce the signature compilation, which retains only actions that contribute positively to conditions. 

Our approach provides a practical and theoretically sound method for applying standard numeric heuristics based on subgoaling in the context of control variables, supported with experimental evaluation. We observed that the over-approximation inherent in the optimistic compilation, which ignores the control component in the decision-making, can lead to decreased performance in some domains.

This work therefore opens a promising line of research for future exploration of heuristic computation in problems with control variables. In particular, by identifying concrete valuations for which possible achievers are more informative in order to bridge the gap with most geometric domains. Although addressing this challenge is nontrivial, our approach provides a solid foundation for developing increasingly informed and effective heuristics in this setting.

\bibliography{aaai2026}

\end{document}

%% file: macros.tex
% Models
\newcommand{\mdp}{\textsc{Mdp}}
\newcommand{\mdps}{\textsc{MDP}s}
\newcommand{\goalmdp}{Goal \mdp}
\newcommand{\goalmdps}{Goal \mdps}
\newcommand{\pomdps}{\textsc{Pomdp}s}
\newcommand{\pomdp}{\textsc{Pomdp}}
\newcommand{\goalpomdps}{Goal \pomdps}
\newcommand{\goalpomdp}{Goal \pomdp}

% Theorems, etc
% \newtheorem{thm}{Theorem}
\newtheorem{proposition}{Proposition}
\newtheorem{corollary}{Corollary}
\newtheorem{example}{Example}
% \newenvironment{proof}{\paragraph{Proof:}}{\hfill$\square$}

% Hector's tricks
\newcommand{\checkit}[1]{}
\newcommand{\ie}{{\textsl{i.e.}}}
\newcommand{\eqdef}{\stackrel{\hbox{\tiny{def}}}{=}}
\newcommand{\pair}[2]{{\langle #1,#2\rangle}}
\newcommand{\denselist}{\itemsep -1pt\partopsep 0pt}

% tuples etc.
\newcommand{\tuple}[1]{\langle #1 \rangle}

%%%%%%%%%%%%%%%%%%%%%%%%%% General Math

\newcommand{\A}{\mathcal{A}} \newcommand{\B}{\mathcal{B}}
\newcommand{\C}{\mathcal{C}} \newcommand{\D}{\mathcal{D}}
\newcommand{\E}{\mathcal{E}} \newcommand{\F}{\mathcal{F}}
\newcommand{\G}{\mathcal{G}} \renewcommand{\H}{\mathcal{H}}
\newcommand{\I}{\mathcal{I}} \newcommand{\J}{\mathcal{J}}
\newcommand{\K}{\mathcal{K}} \renewcommand{\L}{\mathcal{L}}
\newcommand{\M}{\mathcal{M}} \newcommand{\N}{\mathcal{N}}
\renewcommand{\O}{\mathcal{O}} \renewcommand{\P}{\mathcal{P}}
\newcommand{\Q}{\mathcal{Q}} \newcommand{\R}{\mathcal{R}}
\renewcommand{\S}{\mathcal{S}} \newcommand{\T}{\mathcal{T}}
\newcommand{\U}{\mathcal{U}} 
\newcommand{\V}{\mathcal{V}}
\newcommand{\W}{\mathcal{W}} \newcommand{\X}{\mathcal{X}}
\newcommand{\Y}{\mathcal{Y}} \newcommand{\Z}{\mathcal{Z}}

\newcommand{\BE}{{\mathcal{B}, \mathcal{E}}} 

\newcommand{\TAU}{\mathcal{T}} 

\newcommand{\set}[1]{\{#1\}}                      % set
\newcommand{\Set}[1]{\left\{#1\right\}}
\newcommand{\card}[1]{|{#1}|}                     % cardinality of a set
\newcommand{\Card}[1]{\left| #1\right|}
\newcommand{\sub}[1]{[#1]}
\newcommand{\tup}[1]{\langle #1\rangle}            % tuple
\newcommand{\Tup}[1]{\left\langle #1\right\rangle}

\newcommand{\pv}[1]{#1^{b}}
\newcommand{\sv}[1]{#1^{d}}

\newcommand{\hpp}{ \langle \pv{X},\sv{X},D,s_0,A,G,C \rangle }

% % % Names of things, acronyms, software, etc.

\newcommand{\PDDL}{\textsc{pddl}}
\newcommand{\PDDLplus}{\textsc{pddl}$+$}
\newcommand{\PDDLPlus}{\textsc{pddl}$+$}
\newcommand{\FPDDL}{$\smallint$-\textsc{pddl}$+$}
\newcommand{\ADL}{\textsc{adl}}
\newcommand{\FS}{\propername{FS}}
\newcommand{\FSTRIPS}{FSTRIPS}
\newcommand{\SMT}{\textsc{Smt}}
\newcommand{\SAT}{\textsc{Sat}}
\newcommand{\CP}{\textsc{Cp}}
\newcommand{\FMSMT}{\propername{FM-SMT}}
\newcommand{\HSHP}{\propername{HSHP}}
\newcommand{\CNF}{\propername{CNF}}
\newcommand{\Zthree}{\propername{Z3}}
\newcommand{\MathSat}{\propername{MathSAT}}
\newcommand{\LpSat}{\propername{LPSAT}}
\newcommand{\TmLpSat}{\propername{TM-LPSAT}}
\newcommand{\Colin}{\propername{COLIN}}
\newcommand{\PSR}{\textsc{Psr}}
\newcommand{\HBW}{\textsc{Hbw}}
\newcommand{\FF}{\propername{FF}}
\newcommand{\MetricFF}{\propername{MetricFF}}
\newcommand{\FSZero}{\propername{FS0}}
\newcommand{\FSPlus}{\propername{FS+}}
\newcommand{\LTI}{\textsc{Lti}}
\newcommand{\HA}{\textsc{Ha}}
\newcommand{\UPMurphi}{\propername{UPMurphi}}
\newcommand{\Dino}{\propername{DiNo}}
\newcommand{\dreal}{\propername{dReal}}
\newcommand{\Convoys}{\textsc{Convoys}}
\newcommand{\ENHSP}{\propername{ENHSP}}

\newcommand{\Kongming}{\propername{Kongming}}
\newcommand{\GraphPlan}{\propername{GraphPlan}}
\newcommand{\SpaceEx}{\propername{SpaceEx}}
\newcommand{\ode}{\textsc{ode}}
\newcommand{\blind}{\textsc{Brfs}}
\newcommand{\uct}{\textsc{Uct}}
\newcommand{\iwk}{\textsc{IW}($k$)}
\newcommand{\PNE}{\propername{PNE}}
\newcommand{\Fstrips}{\textsc{FSTRIPS}}
\newcommand{\FOL}{\textsc{FOL}}
\newcommand{\MPC}{\textsc{MPC}}
\newcommand{\TFD}{\propername{TFD}}
\newcommand{\AIBR}{\textsc{Aibr}}
\newcommand{\RRT}{\textsc{RRT}}
\newcommand{\SMTPlan}{\textsc{SMTPlan+}}
\newcommand{\TwoLPBlind}{$2$\textsc{LP}--BRFS}
\newcommand{\TwoLPGBFS}{$2$\textsc{LP}--GBFS}
\newcommand{\TwoLPSAT}{$2$\textsc{LP}--SAT}

\newcommand{\pre}{\mathsf{pre}}
\newcommand{\eff}{\mathsf{eff}}
\newcommand{\Pre}{\mathsf{Pre}}
\newcommand{\Eff}{\mathsf{Eff}}
\newcommand{\add}{\mathsf{add}}
\newcommand{\del}{\mathsf{del}}
\newcommand{\action}[1]{\Tup{ \pre{#1}, \eff{#1} }}

\newcommand{\invmatr}[1]{\mathbf{#1}^{-1}}
\newcommand{\transmatr}[1]{\mathbf{\Phi}(#1)}
\newcommand{\thereals}{\ensuremath{\mathbb{R}}}
\newcommand{\thenaturals}{\ensuremath{\mathbb{N}}}
\newcommand{\cmark}{\checkmark}
\newcommand{\xmark}{\mathsf{X}}

\newcommand{\matr}[1]{\ensuremath{\mathbf{#1}}}
\newcommand{\mat}[1]{\ensuremath{\mathbf{#1}}}
\newcommand{\matt}[1]{\ensuremath{\mathbf{#1}^{T}}}
\newcommand{\Useq}{\ensuremath{\mathbf{U}}}
\newcommand{\ubar}[1]{\text{\b{$#1$}}}

%%Domains

\newcommand{\block}[0]{\textsc{Blocks}\xspace}
\newcommand{\mic}[0]{\textsc{Miconic}\xspace}
\newcommand{\driver}[0]{\textsc{Driverlog}\xspace}
\newcommand{\sate}[0]{\textsc{Satellite}\xspace}
\newcommand{\counters}[0]{\textsc{Counters}\xspace}
\newcommand{\blockgrouping}[0]{\textsc{Blocks-grouping}\xspace}
\newcommand{\sailing}[0]{\textsc{Sailing}\xspace}
\newcommand{\drone}[0]{\textsc{Drone}\xspace}
\newcommand{\procurement}[0]{\textsc{Procurement}\xspace}
\newcommand{\cashpoint}[0]{\textsc{Cashpoint}\xspace}
\newcommand{\terraria}[0]{\textsc{Terraria}\xspace}

%%Comments

%% file: aaai2026.bib
@inproceedings{
aso-mollar2025ICAPS,
title={A Sampling Approach to Planning with Infinite Domain Control Variables},
author={{\'A}ngel Aso-Mollar and Diego Aineto and Enrico Scala and Eva Onaindia},
  booktitle = {Proceedings of the Thirty-Fifth International  Conference on
               Automated Planning and Scheduling, {ICAPS-25}},
  pages     = {149--153},
  year      = {2025},

}

@inproceedings{
aso-mollar2025IJCAI,
title={Handling Infinite Domain Parameters in Planning Through Best-First Search with Delayed Partial Expansions},
author={{\'A}ngel Aso-Mollar and Diego Aineto and Enrico Scala and Eva Onaindia},
  booktitle = {Proceedings of the Thirty-Fourth International Joint Conference on
               Artificial Intelligence, {IJCAI-25}},
  publisher = {International Joint Conferences on Artificial Intelligence Organization},
  pages     = {8456--8464},
  year      = {2025},
}

@article{subgoaling20,
author = {Scala, Enrico and Haslum, Patrik and Thiebaux, Sylvie and Ramirez, Miquel},
year = {2020},
month = {08},
pages = {691-752},
title = {Subgoaling Techniques for Satisficing and Optimal Numeric Planning},
volume = {68},
journal = {Journal of Artificial Intelligence Research},
doi = {10.1613/jair.1.11875}
}

@article{bonet01,
title = {Planning as heuristic search},
journal = {Artificial Intelligence},
volume = {129},
number = {1},
pages = {5-33},
year = {2001},
issn = {0004-3702},
doi = {https://doi.org/10.1016/S0004-3702(01)00108-4},
url = {https://www.sciencedirect.com/science/article/pii/S0004370201001084},
author = {Blai Bonet and Héctor Geffner},
}

@inproceedings{scala18,
  title     = {Effect-Abstraction Based Relaxation for Linear Numeric Planning},
  author    = {Dongxu Li and Enrico Scala and Patrik Haslum and Sergiy Bogomolov},
  booktitle = {Proceedings of the Twenty-Seventh International Joint Conference on
               Artificial Intelligence, {IJCAI-18}},
  publisher = {International Joint Conferences on Artificial Intelligence Organization},
  pages     = {4787--4793},
  year      = {2018},
}

@article{FoxL03,
  author       = {Maria Fox and
                  Derek Long},
  title        = {{PDDL2.1:} An Extension to {PDDL} for Expressing Temporal Planning Domains},
  journal      = {Journal of Artificial Intelligence Research},
  volume       = {20},
  pages        = {61--124},
  year         = {2003}
}

@article{ShinD05,
  author       = {Ji{-}Ae Shin and
                  Ernest Davis},
  title        = {{Processes and Continuous Change in a SAT-based Planner}},
  journal      = {Artificial Intelligence},
  volume       = {166},
  number       = {1-2},
  pages        = {194--253},
  year         = {2005}
}

@inproceedings{SavasFLM16,
  author       = {Emre Savaş and
                  Maria Fox and
                  Derek Long and
                  Daniele Magazzeni},
  title        = {{Planning Using Actions with Control Parameters}},
  booktitle    = {Procedings of the Twenty-Second European Conference on Artificial Intelligence, ECAI-16},
  series       = {Frontiers in Artificial Intelligence and Applications},
  volume       = {285},
  pages        = {1185--1193},
  publisher    = {{IOS} Press},
  year         = {2016}
}

@inproceedings{heesch24,
  author       = {René Heesch and Alessandro Cimatti and Jonas Ehrhardt and Alexander Diedrich and Oliver Niggemann},
  title        = {{A Lazy Approach to Neural Numerical Planning with Control Parameters}},
  booktitle    = {{European Conference on Artificial Intelligence} 2024},
  series       = {Frontiers in Artificial Intelligence and Applications},
  pages        = {4262--4270},
  publisher    = {{IOS} Press},
  year         = {2024}
}

@article{nextflap24,
title = {A Hybrid Approach for Expressive Numeric and Temporal Planning with Control Parameters},
journal = {Expert Systems with Applications},
volume = {242},
pages = {122820},
year = {2024},
issn = {0957-4174},
doi = {https://doi.org/10.1016/j.eswa.2023.122820},
url = {https://www.sciencedirect.com/science/article/pii/S0957417423033225},
author = {Oscar Sapena and Eva Onaindia and Eliseo Marzal},
}

@inproceedings{Scala2016IntervalBasedRF,
  title={Interval-Based Relaxation for General Numeric Planning},
  author={Enrico Scala and Patrik Haslum and Sylvie Thi{\'e}baux and Miquel Ram{\'i}rez},
  booktitle={Proceedings of the Twenty-Second European Conference on Artificial Intelligence},
  year={2016},
  url={https://api.semanticscholar.org/CorpusID:27984436}, 
  pages={655--663},
  publisher    = {{IOS} Press},
}

@inproceedings{htncompilation,
  title     = {On Guiding Search in HTN Planning with Classical Planning Heuristics},
  author    = {Höller, Daniel and Bercher, Pascal and Behnke, Gregor and Biundo, Susanne},
  booktitle = {Proceedings of the Twenty-Eighth International Joint Conference on
               Artificial Intelligence, {IJCAI-19}},
  publisher = {International Joint Conferences on Artificial Intelligence Organization},
  pages     = {6171--6175},
  year      = {2019},
  month     = {7},
  doi       = {10.24963/ijcai.2019/857},
  url       = {https://doi.org/10.24963/ijcai.2019/857},
}

@article{hmrp, title={Search-Guidance Mechanisms for Numeric Planning Through Subgoaling Relaxation }, volume={30}, url={https://ojs.aaai.org/index.php/ICAPS/article/view/6665}, DOI={10.1609/icaps.v30i1.6665}, abstractNote={&lt;p&gt;Recently, a new decomposition based relaxation has been proposed for numeric planning problems. Roughly, this relaxation is grounded on the identification of regression-based necessary conditions for the satisfaction of sets of numeric subgoals. So far, it has been used to define novel heuristics that are able to provide great guidance in problems exhibiting a pronounced numeric structure. This paper investigates how to further exploit this relaxation; it does so by introducing the notion of the multi-repetition relaxed plan. The multi-repetition plan annotates actions with the number of times such actions need to be executed. We use this structure for different purposes: extraction of a concrete relaxed plan based heuristic, definition of subgoaling based helpful actions, and definition of what we call up-to-jumping actions. Up-to-jumping actions allow us to deeply leverage from the metric structure of the problem and devise an informed search strategy that can collapse several decision steps. We experimentally analyze a forward state space planner equipped with these novel mechanisms across several planning benchmarks, showing the benefit of the ideas presented in the paper.&lt;/p&gt;}, number={1}, journal={Proceedings of the Thirtieth International Conference on Automated Planning and Scheduling}, author={Scala, Enrico and Saetti, Alessandro and Serina, Ivan and Gerevini, Alfonso E.}, year={2020}, month={Jun.}, pages={226-234} }
